\documentclass[letterpaper]{article} 
\usepackage{aaai2027}  
\usepackage[hyphens]{url}  
\usepackage{graphicx} 
\usepackage{natbib}  
\usepackage{caption} 
\usepackage{booktabs}
\usepackage{multirow}
\usepackage{amsmath}
\usepackage{amssymb}
\usepackage{amsthm}
\usepackage{bm}
\usepackage{algorithm}
\usepackage{algorithmic}

\usepackage{xr}
\nocopyright

\newtheorem{theorem}{Theorem}
\newtheorem{assumption}{Assumption}

\newcommand{\E}{\mathbb{E}}
\newcommand{\ind}{\mathbb{I}}

\title{When Interference Graphs Evolve: Doubly Robust Estimation of Dynamic Peer Effects}
\author{
    Xiaojing Du
}
\affiliations{
    Adelaide University
}

\begin{document}

\maketitle

\begin{abstract}
Peer effects are difficult to estimate when interaction graphs evolve because pre-assignment network history, dynamic peer exposure, and post-assignment network change have distinct causal roles. We introduce a controlled contrast framework that indexes potential outcomes by own treatment, temporally aggregated peer exposure, and a post-assignment evolution summary. Differences between the resulting means define own-treatment, peer-exposure, controlled network-evolution, and joint controlled contrasts rather than a mediation decomposition. We develop the Dynamic Network Doubly Robust estimator, DynaNet-DR, which combines a temporally factorized propensity with normalized augmentation. Under consistency, summary sufficiency, sequential exchangeability, positivity, nuisance convergence, and weak dependence, its canonical estimator is consistent when either the outcome regression or the propensity estimator is consistent. The reported implementation adds representative-score prediction, fixed clipping, and finite-sample stabilization. Semi-synthetic benchmarks on fixed real temporal graph sequences show favorable estimation accuracy among methods targeting the full profile. These benchmarks assess summary-indexed contrasts rather than counterfactual edge generation, and the MathOverflow study is an observational illustration under the stated assumptions.
\end{abstract}

\section{Introduction}

Interventions on social platforms, recommender systems, educational platforms, and digital marketplaces can affect treated individuals and their peers. When a user receives a promotion, recommendation, incentive, or collaboration tool, outcomes may also change among friends, teammates, classmates, or trading partners. Such peer effects violate the no-interference assumption of classical causal inference \citep{manski1993identification,hudgens2008toward,aronow2017estimating}. Because experimental and observational evaluations inform rollout and targeting decisions, analyses that ignore interference can conflate responses to a unit's own assignment with responses to peer assignments \citep{eckles2017design,forastiere2021identification}. The central question is how to define and estimate these responses when the interaction graph changes over time.

Platform networks evolve as users follow, comment, form teams, and trade. A user's neighborhood can change across periods, and brief interactions can affect later peer exposure. Static summaries based only on the contemporaneous fraction of treated neighbors may omit relevant temporal structure \citep{aronow2017estimating,savje2024causal,adhikari2025learning,lin2026modeling}. Dynamic exposure can depend on current and historical neighbors, edge strength, and temporal persistence. We therefore construct a structured exposure mapping with interpretable intervention levels and diagnose their empirical support.

Post-assignment network evolution introduces a separate causal role. Invitation and recommendation policies may alter subsequent interactions or local connectivity, while other evolution may occur independently of treatment. Pre-assignment graph history informs confounding adjustment and exposure construction; a post-assignment evolution summary defines an additional coordinate of the controlled intervention. Varying this summary while fixing own treatment and peer exposure yields a controlled network-evolution contrast, not a natural indirect effect or a component of a total-effect decomposition. The framework permits externally occurring or treatment-responsive evolution, but the reported benchmarks keep each observed graph sequence fixed and generate the summary as a post-assignment label; they do not generate counterfactual edges. Figure~\ref{fig:dag} previews the four controlled comparisons in natural language. Within each panel, the target differs from the reference only in the highlighted feature, and the comparisons are not additive pathways.

Prior work addresses related components of this problem. Exposure mappings and generalized propensity scores formalize interference for a specified graph \citep{hudgens2008toward,aronow2017estimating,forastiere2021identification}. Graph-learning methods estimate high-dimensional exposure representations \citep{ma2021causal,adhikari2025learning}, and recent studies develop doubly robust estimators for network interference \citep{chen2024doubly,khatami2025graph}. Dynamic-interference models incorporate evolving graphs into prediction \citep{lin2026modeling}; treatment-responsive-network and edge-intervention studies target related but distinct changes to graph structure \citep{comola2021treatment,chen2026connections}. Our framework uses pre-assignment graph history for adjustment and exposure construction, then includes a specified post-assignment evolution summary in the controlled mean profile.

This paper studies causal effect estimation under evolving network interference. We distinguish pre-assignment network history, used for adjustment and exposure construction, from a post-assignment network evolution summary, treated as a coordinate of the controlled intervention. We develop the Dynamic Network Doubly Robust estimator, DynaNet-DR, for the resulting controlled contrast profile. It combines temporal graph features, a temporally ordered propensity factorization, and normalized augmentation. We make four contributions.
\begin{itemize}
    \item We formalize evolving-network interference through a controlled contrast profile that separates pre-assignment network history, dynamic peer exposure, and a post-assignment evolution summary.
    \item We construct a dynamic exposure mapping from recent peer assignments, pre-assignment edge weights, and temporal decay, with interpretable intervention levels and explicit overlap diagnostics.
    \item We develop DynaNet-DR by combining a temporally ordered propensity factorization with normalized augmentation. We establish double robustness for its canonical form and distinguish that result from the reported finite-sample stabilization.
    \item We evaluate same-estimand accuracy and restricted-design scope on four fixed real temporal graph sequences with generated assignments, evolution-summary labels, and outcomes; isolate the adaptive gate in an ablation; report overlap and hidden-homophily diagnostics; and present MathOverflow as an observational illustration.
\end{itemize}

\begin{figure}[t]
\centering
\includegraphics[width=0.98\columnwidth]{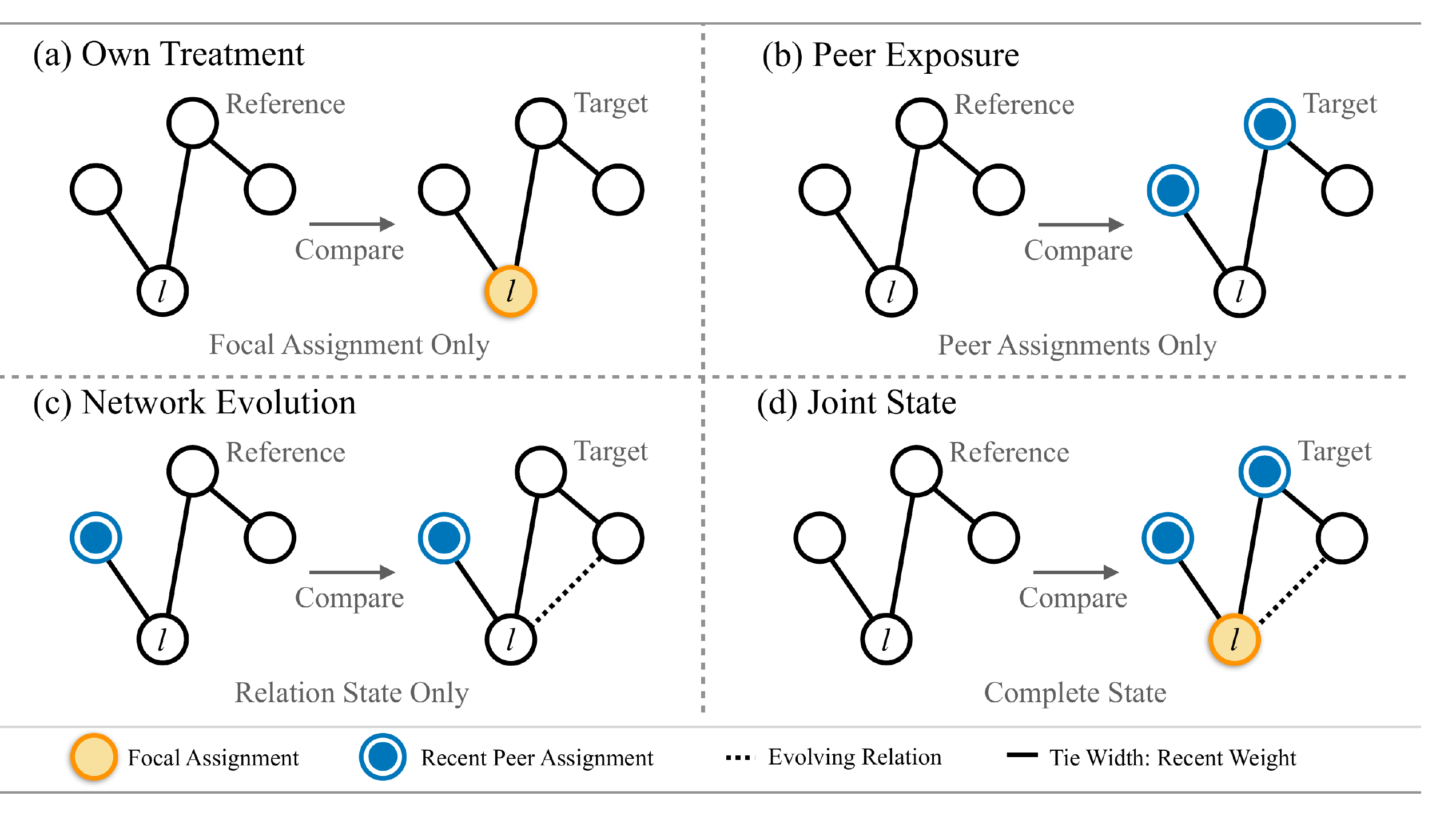}
\caption{Controlled comparisons for focal unit $i$. Each panel varies one highlighted feature while holding the others fixed.}
\label{fig:dag}
\end{figure}

\section{Problem Definition}

We define the dynamic networked observation model, the exposure and evolution summaries, the target mean, three coordinate-wise controlled contrasts, and a joint controlled contrast. Appendix~\ref{app:notation} summarizes the notation.

\paragraph{Dynamic networked observations.}
Consider $n$ units observed over $T$ time periods. At assignment time $t$, we observe a pre-assignment interaction graph $G^{-}_{t}$, a binary treatment vector $\bm{A}_t$ with $A_{it} \in \{0,1\}$, a subsequent graph $G^{+}_{t+1}$, and an outcome vector $\bm{Y}_{t+1}$ measured after the evolution summary. For each unit-time pair $(i,t)$, let $H_{it}$ contain observed covariates, treatments, outcomes, and network structure available before assignment. The index $t+1$ denotes the post-assignment target horizon rather than requiring evolution and outcome to be measured simultaneously; Section~5 and Appendix~\ref{app:impl} give the raw-bin mappings for the case study and benchmarks.

\paragraph{Peer exposure.}
The treatment vector of all other units is high-dimensional, so a dynamic exposure variable $Z_{it}$ aggregates treatments of relevant peers using recent network history, edge strength, and temporal decay. We use discrete levels $Z_{it} \in \{0, 1, \dots, K\}$ to define the target interventions and assess their empirical support.

\paragraph{Network evolution summary.}
Let $M_{i,t+1}$ summarize local graph evolution after assignment, such as tie formation, dissolution, or intensity change. Unlike pre-assignment graph history, $M_{i,t+1}$ is not included as an ordinary confounder; it is a coordinate of the controlled intervention. We take $M_{i,t+1} \in \{0,1\}$. The summary may describe externally occurring or treatment-responsive evolution, but the reported benchmark keeps the observed graph sequence fixed and generates $M$ as a label, as detailed in Section~5.

\paragraph{Potential outcomes.}
Let $Y_{i,t+1}(a,z,m)$ denote the potential outcome of unit $i$ when its own treatment, peer exposure, and network evolution summary are set to $(a,z,m)$. We target the marginal mean potential outcome
\begin{equation}
\mu(a,z,m) = \E\left[ Y_{i,t+1}(a,z,m) \right],
\label{eq:mu}
\end{equation}
where the expectation is over the target distribution of unit-times and pre-assignment histories. Equation~\eqref{eq:mu} defines a single-world controlled mean. Natural, interventional, and path-specific mediation effects require additional counterfactual structure.

\paragraph{Causal contrasts.}
The target mean $\mu(a,z,m)$ supports three coordinate-wise controlled contrasts. The \emph{own-treatment contrast} compares own-treatment levels while holding peer exposure and network evolution fixed:
\begin{equation}
\tau_D(z,m) = \mu(1,z,m) - \mu(0,z,m).
\label{eq:de}
\end{equation}
The \emph{peer-exposure contrast} compares peer-exposure levels while holding own treatment and network evolution fixed:
\begin{equation}
\tau_S(a; z, z', m) = \mu(a,z,m) - \mu(a,z',m).
\label{eq:pe}
\end{equation}
The \emph{controlled network-evolution contrast} compares post-assignment evolution-summary levels while holding own treatment and peer exposure fixed:
\begin{equation}
\tau_M(a,z; m,m') = \mu(a,z,m) - \mu(a,z,m').
\label{eq:nme}
\end{equation}
For two joint intervention levels $\theta_1=(a_1,z_1,m_1)$ and $\theta_0=(a_0,z_0,m_0)$, write $\mu(\theta_j)=\mu(a_j,z_j,m_j)$. We define the joint controlled contrast $\tau_J(\theta_1,\theta_0)=\mu(\theta_1)-\mu(\theta_0)$. Together, the four contrast families form a profile of distinct state comparisons without imposing an additive decomposition. The estimation task is to recover these means and contrasts under observed time-varying confounding and network dependence, with explicit diagnostics for limited overlap. Homophily left unmeasured in $H_{it}$ can violate sequential exchangeability.

\section{Identification}

We state the identification assumptions and derive the observed-data functional for the controlled means. Section~4 develops the estimator and states its theoretical scope.

\subsection{Assumptions}

\begin{assumption}[Consistency]
\label{as:consistency}
For each unit-time pair, the observed outcome equals the potential outcome indexed by the observed own treatment, peer exposure, and network evolution summary.
\end{assumption}

\begin{assumption}[Exposure and evolution-summary sufficiency]
\label{as:sufficiency}
Conditional on $H_{it}$, the potential outcome at the target horizon depends on other units' assignments and dynamic network paths only through $A_{it}$, $Z_{it}$, and $M_{i,t+1}$. Graph configurations mapped to the same $(z,m)$ are treated as outcome-equivalent versions.
\end{assumption}

\begin{assumption}[Sequential exchangeability]
\label{as:exchangeability}
Conditional on $H_{it}$, the joint assignment $(A_{it}, Z_{it})$ is independent of the potential outcomes $\{Y_{i,t+1}(a,z,m)\}$. In addition, conditional on $(H_{it}, A_{it}, Z_{it})$, the observed network evolution summary $M_{i,t+1}$ is independent of the controlled potential outcomes $\{Y_{i,t+1}(a,z,m)\}$.
\end{assumption}

\begin{assumption}[Positivity]
\label{as:positivity}
Every treatment, peer exposure, and network evolution level used in the target estimand has positive probability conditional on histories in the target population: the joint treatment-exposure probability and the conditional network-evolution probability are bounded away from zero for target levels.
\end{assumption}

Assumptions~\ref{as:consistency} through \ref{as:positivity} follow sequential identification arguments \citep{robins1986new,hernan2020causal} and their network counterparts under exposure mappings \citep{hudgens2008toward,forastiere2021identification}. Assumption~\ref{as:sufficiency} specifies the summaries and their outcome-equivalent versions, which is substantive when exposure mappings may be misspecified \citep{savje2024causal}. Stochastic interventions over graph configurations provide a possible extension. The second part of Assumption~\ref{as:exchangeability} addresses confounding of the post-assignment summary.

Figure~\ref{fig:causaldag} summarizes the assumed structure. Under Assumption~\ref{as:exchangeability}, $H_{it}$ is sufficient for the two stated sequential exchangeability conditions; in the default benchmark, its nuisance representation includes $U_i$. Summary sufficiency and exchangeability are substantive assumptions, not consequences of the representation.

\paragraph{Where homophily enters.}
In the benchmark, homophily-related confounding is represented by a generated community-correlated unit trait $U_i$. The trait enters the treatment, evolution-label, and outcome models, while its community correlation associates it with peer exposure and the fixed graph structure. The default benchmark includes $U_i$ in the nuisance features; Appendix~\ref{app:hidden} reports a sensitivity analysis that withholds it. As in observational network studies more generally, peer correlation is not separable from contagion without additional assumptions \citep{shalizi2011homophily}.

\begin{figure}[t]
\centering
\includegraphics[width=0.80\columnwidth]{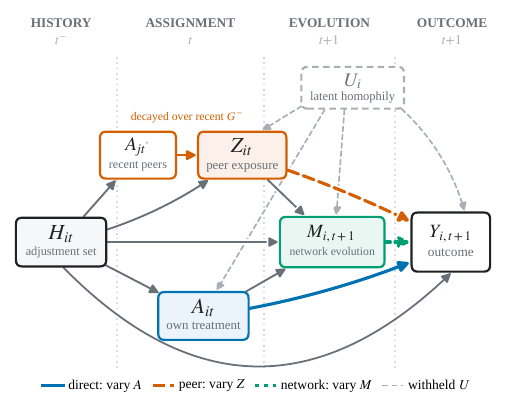}
\caption{Time ordering. Peer treatments are aggregated over pre-assignment graphs into $Z_{it}$. The post-assignment summary $M_{i,t+1}$ precedes the outcome. The generated trait $U_i$ is included by default and withheld in the sensitivity analysis.}
\label{fig:causaldag}
\end{figure}

\subsection{Identification}

\begin{theorem}[Identification of controlled means]
\label{thm:identification}
Under Assumptions~\ref{as:consistency} through \ref{as:positivity}, the mean potential outcome $\mu(a,z,m)$ is identified by
\begin{equation}
\begin{split}
\mu(a,z,m) = \E\bigl[ \E[ Y_{i,t+1} \mid {}& A_{it}=a, Z_{it}=z, \\
& M_{i,t+1}=m, H_{it} ] \bigr].
\end{split}
\label{eq:identification}
\end{equation}
Consequently, all contrasts of $\mu(a,z,m)$, including Eqs.~\eqref{eq:de} to \eqref{eq:nme}, are identifiable.
\end{theorem}

\begin{proof}
Consistency connects observed and potential outcomes, sequential exchangeability removes confounding for $(A_{it}, Z_{it})$ and then $M_{i,t+1}$, and positivity makes the regression well defined at every target level. Averaging over histories yields Eq.~\eqref{eq:identification}, as derived in full in Appendix~\ref{app:proof}.
\end{proof}

\section{Method}

The Dynamic Network Doubly Robust estimator, DynaNet-DR, combines a structured dynamic exposure, a post-assignment evolution summary, temporally held-out nuisance estimation, and normalized augmentation. Algorithm~\ref{alg:dynanet} in Appendix~\ref{app:impl} summarizes the reported pipeline.

\subsection{Dynamic Exposure Representation}

The reported exposure module uses fixed temporal decay and edge weights to aggregate recent peer treatments into a continuous score $S_{it}$, then discretizes it into low, medium, and high levels $Z_{it}$. The canonical estimator below treats $Z$ as the intervention and uses $Q(a,z,m,h)$. In the reported implementation, $S_{it}$ is also supplied to the outcome learner: observed-outcome residuals use the observed score, whereas a target prediction at level $z$ sets the score to its training conditional mean $s_z$. We refer to this as a \emph{representative-score plug-in approximation}. It is not generally equivalent to integrating a nonlinear outcome regression over the within-level score distribution. The benchmark outcome is generated from discrete $Z$ rather than $S$, but this score-irrelevance property is not assumed for the observational application.

\subsection{Network Evolution Representation}

In the general framework, $M_{i,t+1}$ may summarize a difference between post- and pre-assignment local graphs. The semi-synthetic benchmark instead generates a binary post-assignment label conditional on the simulated variables while keeping the real graph sequence fixed; the MathOverflow application uses an indicator for gaining a previously unseen neighbor. In each case, $M$ enters the target and nuisance models as a controlled level, not as a pre-assignment adjustment variable. The estimand does not by itself identify the effect of a particular edge-editing policy.

\subsection{Nuisance Models with Temporal Graph Features}

We fit gradient-boosted outcome and propensity models to handcrafted node-time features $\phi_{it}$ from $H_{it}$. The outcome model conditions on $\phi_{it}$ and $(a,z,m)$; the propensity models estimate the joint assignment $g_{AZ,0}(a,z\mid H_{it})$ and the subsequent summary mechanism $g_{M,0}(m\mid a,z,H_{it})$. The training window is divided into three contiguous blocks. Each nuisance model is fit on the complement of one block, and the three predictions are averaged on a later estimation window not used for fitting. This is a temporally held-out nuisance ensemble rather than observation-level cross-fitting of the estimation sample; theoretical consistency therefore requires nuisance convergence on the estimation-window distribution.

\subsection{Canonical Estimator and Reported Stabilization}

Let
\begin{equation}
\begin{split}
Q_0(a,z,m,h) = \E[ Y_{i,t+1} \mid {}& A_{it}=a, Z_{it}=z, \\
& M_{i,t+1}=m, H_{it}=h ].
\end{split}
\label{eq:outcome}
\end{equation}
Equations~\eqref{eq:outcome} through \eqref{eq:dr} use the canonical discrete-$Z$ regression $\widehat Q$. For the reported implementation, let $\widehat Q^S(a,z,m,h,s)$ denote the score-augmented learner. Its target prediction substitutes $\widehat Q^S(a,z,m,H_{it},s_z)$ for $\widehat Q(a,z,m,H_{it})$, whereas the observed residual uses $\widehat Q^S(A_{it},Z_{it},M_{i,t+1},H_{it},S_{it})$. These substitutions define the representative-score statistic and coincide with the canonical estimator only under the compatibility condition stated below. For $\theta=(a,z,m)$, write $\widehat Q(\theta,h)=\widehat Q(a,z,m,h)$. We factor the generalized propensity in temporal order,
\begin{equation}
g_0(a,z,m \mid h) = g_{AZ,0}(a,z \mid h) \, g_{M,0}(m \mid a,z,h).
\label{eq:propensity}
\end{equation}
Define
\begin{equation}
\widehat{W}_{\theta}(i,t) =
\frac{\ind\{(A_{it},Z_{it},M_{i,t+1})=\theta\}}
{\widehat g(\theta\mid H_{it})}.
\label{eq:weights}
\end{equation}
For the canonical estimator, define
\[
\widehat Q^{\mathrm{obs}}_{it}
=\widehat Q(A_{it},Z_{it},M_{i,t+1},H_{it}).
\]
For the representative-score statistic, instead define
\[
\widehat Q^{\mathrm{obs}}_{it}
=\widehat Q^S(A_{it},Z_{it},M_{i,t+1},H_{it},S_{it}).
\]
In either case, let $R_{it}=Y_{i,t+1}-\widehat Q^{\mathrm{obs}}_{it}$. The normalized correction is
\begin{equation}
\widehat c(\theta)=
\frac{\sum_{(i,t)}\widehat W_{\theta}(i,t)R_{it}}
{\sum_{(i,t)}\widehat W_{\theta}(i,t)}.
\label{eq:corr}
\end{equation}
For target levels $\theta_1,\theta_0$ and $N$ estimation-window unit-time records, we write
\begin{equation}
\begin{split}
\widehat\tau(\theta_1,\theta_0)={}&
\frac{1}{N}\sum_{(i,t)}
\{\widehat Q(\theta_1,H_{it})-\widehat Q(\theta_0,H_{it})\}\\
&+\widehat\lambda\{\widehat c(\theta_1)-\widehat c(\theta_0)\}.
\end{split}
\label{eq:dr}
\end{equation}
The \emph{canonical} normalized augmented estimator sets $\widehat\lambda=1$ and uses the unclipped estimated propensity. It is the estimator covered by Theorem~\ref{thm:dr}.

The reported implementation first clips the estimated joint propensity to $[0.02,0.98]$. For each target cell, the inverse-weight effective sample size is $(\sum_{i,t}\widehat W_\theta(i,t))^2/\sum_{i,t}\widehat W_\theta(i,t)^2$; $\widehat n_{\mathrm{eff}}$ is the smaller of the two cell-specific values. This diagnostic measures weight concentration, not the number of independent observations under network dependence. Let $\widehat d=\widehat c(\theta_1)-\widehat c(\theta_0)$ and let $\widehat V(\widehat d)$ be its node-clustered variance estimate. With $n_{\min}^{\ast}=60$, the implementation sets
\begin{equation}
\widehat\lambda=\frac{\widehat d^{\,2}}
{\widehat d^{\,2}+\kappa\widehat V(\widehat d)},\qquad
\kappa=\max\{4,16n_{\min}^{\ast}/\widehat n_{\mathrm{eff}}\},
\label{eq:gate}
\end{equation}
then sets $\widehat\lambda=0$ when $\widehat n_{\mathrm{eff}}<n_{\min}^{\ast}$ or the displayed value is below $0.25$. These operations are finite-sample stabilization choices.

\begin{theorem}[Double robustness of the canonical estimator]
\label{thm:dr}
Consider Eq.~\eqref{eq:dr} with $\widehat\lambda=1$ and an unclipped propensity estimate, or with clipping that is asymptotically inactive. Suppose Assumptions~\ref{as:consistency} through \ref{as:positivity} hold. Both temporally held-out nuisances converge in $L_2$ on the estimation-window distribution to finite limits $\overline Q$ and $\overline g$, with either $\overline Q=Q_0$ or $\overline g=g_0$. For each target cell, the estimated inverse weights and the terms in Eqs.~\eqref{eq:corr} through \eqref{eq:dr} are uniformly square-integrable, and the normalized denominators converge to positive limits. Assume further that, conditional on the nuisance-training data, estimation-window moments converge to their target-population counterparts and that a weak-dependence law of large numbers holds for the data-adaptive node-time terms. Then $\widehat\tau(\theta_1,\theta_0)$ is consistent for $\mu(\theta_1)-\mu(\theta_0)$.
\end{theorem}

Inverse-weight square integrability follows, for example, if the estimated target-cell propensity is bounded away from zero with probability tending to one. A sufficient within-window dependence condition is that, conditional on nuisance-training data, the $N$ node-time records admit a dependency graph with maximum degree $\Delta_N=o(N)$ and uniformly bounded second moments. The empirical weight-based effective sample size summarizes weight concentration; it does not establish temporal stability or this dependence condition.

\begin{proof}
If $\widehat Q\to Q_0$, the plug-in term converges to the target contrast and each normalized residual correction converges to zero. If $\widehat g\to g_0$, the normalized denominator converges to one and $\widehat c(\theta)$ converges to the target-cell bias correction $\E\{Q_0(\theta,H)-\overline Q(\theta,H)\}$, so the difference of corrections removes the plug-in contrast bias when $\widehat\lambda=1$ \citep{bang2005doubly,vanderlaan2006targeted}.
\end{proof}

For the adaptive estimator, suppose both target-cell effective sample sizes diverge, $\widehat V(\widehat d)\to0$, and clipping becomes asymptotically inactive. When the limiting correction difference is nonzero, the hard gates are eventually inactive and $\widehat\lambda\to1$; when it is zero, the plug-in contrast has the correct limit under either nuisance-correct branch. Thus the adaptive estimator shares the canonical limit under these conditions. The theorem does not cover the reported active fixed clip or the representative-score approximation without an additional within-level score-irrelevance or correct-marginalization condition.

\paragraph{Inference.}
For DynaNet-DR, we report standard errors from the empirical variance of the unshrunk influence scores with node-level clustering. The resulting intervals are nominal: node clustering does not establish a central limit theorem for cross-node dependence induced by an evolving graph \citep{ogburn2024causal,leung2020treatment}. We therefore report their empirical coverage in the semi-synthetic benchmarks without claiming general coverage validity.

Before interpreting a target contrast, we inspect cell counts, estimated propensities, and effective sample sizes. These diagnostics indicate observed support; they do not verify conditional positivity.

\section{Experiments}

We study same-estimand accuracy (\textbf{RQ1}), the estimand scope of restricted designs and ablations (\textbf{RQ2}), nuisance-model stress tests (\textbf{RQ3}), and empirical coverage of the reported nominal intervals (\textbf{RQ4}).

\begin{table*}[t]
\centering
\footnotesize
\setlength{\tabcolsep}{1.5mm}
\begin{tabular}{llccccc}
\toprule
Dataset & Method & DE & PE & CNE & JCC & Coverage \\
\midrule
\multirow{8}{*}{CollegeMsg} & Naive-A & 0.722$\pm$0.042 & 0.552$\pm$0.000 & 0.326$\pm$0.000 & 0.157$\pm$0.043 & 0.00 \\
 & Static-Z OR & 0.099$\pm$0.043 & 0.354$\pm$0.029 & 0.326$\pm$0.000 & 0.591$\pm$0.055 & 0.01 \\
 & Static-Z DR & 0.103$\pm$0.024 & 0.279$\pm$0.051 & 0.326$\pm$0.000 & 0.503$\pm$0.054 & 0.10 \\
 & GPS$^\dagger$ & 0.147$\pm$0.025 & 0.417$\pm$0.044 & 0.326$\pm$0.000 & 0.878$\pm$0.056 & 0.00 \\
 & DynInt$^\dagger$ & 0.069$\pm$0.032 & 0.062$\pm$0.026 & 0.326$\pm$0.000 & 0.254$\pm$0.053 & 0.04 \\
 & TL$^\dagger$ & 0.056$\pm$0.021 & 0.053$\pm$0.034 & 0.049$\pm$0.028 & 0.044$\pm$0.029 & 0.17 \\
 & GML-DR$^\dagger$ & 0.052$\pm$0.029 & \textbf{0.051$\pm$0.028} & 0.044$\pm$0.032 & 0.055$\pm$0.033 & 0.87 \\
 & DynaNet-DR & \textbf{0.049$\pm$0.017} & \textbf{0.051$\pm$0.027} & \textbf{0.038$\pm$0.016} & \textbf{0.028$\pm$0.020} & 0.91 \\
\midrule
\multirow{8}{*}{email-Eu} & Naive-A & 0.812$\pm$0.077 & 0.552$\pm$0.000 & 0.326$\pm$0.000 & 0.083$\pm$0.056 & 0.06 \\
 & Static-Z OR & 0.089$\pm$0.040 & 0.310$\pm$0.028 & 0.326$\pm$0.000 & 0.558$\pm$0.044 & 0.00 \\
 & Static-Z DR & 0.097$\pm$0.033 & 0.300$\pm$0.023 & 0.326$\pm$0.000 & 0.540$\pm$0.043 & 0.23 \\
 & GPS$^\dagger$ & 0.061$\pm$0.013 & 0.351$\pm$0.036 & 0.326$\pm$0.000 & 0.689$\pm$0.045 & 0.01 \\
 & DynInt$^\dagger$ & 0.068$\pm$0.039 & 0.039$\pm$0.027 & 0.326$\pm$0.000 & 0.250$\pm$0.030 & 0.11 \\
 & TL$^\dagger$ & 0.079$\pm$0.041 & 0.092$\pm$0.050 & 0.089$\pm$0.043 & 0.082$\pm$0.041 & 0.07 \\
 & GML-DR$^\dagger$ & 0.083$\pm$0.043 & 0.076$\pm$0.028 & 0.079$\pm$0.034 & 0.059$\pm$0.033 & 0.99 \\
 & DynaNet-DR & \textbf{0.026$\pm$0.019} & \textbf{0.064$\pm$0.038} & \textbf{0.042$\pm$0.017} & \textbf{0.054$\pm$0.047} & 1.00 \\
\midrule
\multirow{8}{*}{PrimarySchool} & Naive-A & 0.633$\pm$0.120 & 0.552$\pm$0.000 & 0.326$\pm$0.000 & 0.247$\pm$0.121 & 0.03 \\
 & Static-Z OR & 0.088$\pm$0.047 & 0.276$\pm$0.076 & 0.326$\pm$0.000 & 0.525$\pm$0.108 & 0.04 \\
 & Static-Z DR & 0.089$\pm$0.042 & 0.267$\pm$0.067 & 0.326$\pm$0.000 & 0.536$\pm$0.109 & 0.41 \\
 & GPS$^\dagger$ & 0.267$\pm$0.074 & 0.514$\pm$0.042 & 0.326$\pm$0.000 & 1.092$\pm$0.085 & 0.00 \\
 & DynInt$^\dagger$ & 0.080$\pm$0.041 & 0.156$\pm$0.084 & 0.326$\pm$0.000 & 0.263$\pm$0.176 & 0.04 \\
 & TL$^\dagger$ & 0.165$\pm$0.087 & 0.203$\pm$0.148 & 0.208$\pm$0.153 & 0.217$\pm$0.228 & 0.09 \\
 & GML-DR$^\dagger$ & 0.146$\pm$0.063 & 0.176$\pm$0.094 & 0.192$\pm$0.131 & 0.180$\pm$0.141 & 0.93 \\
 & DynaNet-DR & \textbf{0.064$\pm$0.035} & \textbf{0.147$\pm$0.074} & \textbf{0.096$\pm$0.044} & \textbf{0.154$\pm$0.125} & 0.99 \\
\midrule
\multirow{8}{*}{HighSchool} & Naive-A & 0.784$\pm$0.101 & 0.552$\pm$0.000 & 0.326$\pm$0.000 & 0.123$\pm$0.060 & 0.09 \\
 & Static-Z OR & 0.099$\pm$0.049 & 0.309$\pm$0.075 & 0.326$\pm$0.000 & 0.573$\pm$0.105 & 0.01 \\
 & Static-Z DR & 0.118$\pm$0.066 & 0.267$\pm$0.042 & 0.326$\pm$0.000 & 0.474$\pm$0.116 & 0.31 \\
 & GPS$^\dagger$ & 0.341$\pm$0.056 & 0.492$\pm$0.057 & 0.326$\pm$0.000 & 1.147$\pm$0.067 & 0.00 \\
 & DynInt$^\dagger$ & 0.062$\pm$0.038 & 0.115$\pm$0.094 & 0.326$\pm$0.000 & 0.244$\pm$0.118 & 0.10 \\
 & TL$^\dagger$ & 0.113$\pm$0.054 & 0.102$\pm$0.052 & 0.106$\pm$0.063 & 0.096$\pm$0.057 & 0.20 \\
 & GML-DR$^\dagger$ & 0.106$\pm$0.070 & 0.091$\pm$0.059 & 0.137$\pm$0.081 & 0.085$\pm$0.054 & 0.94 \\
 & DynaNet-DR & \textbf{0.067$\pm$0.024} & \textbf{0.087$\pm$0.058} & \textbf{0.071$\pm$0.039} & \textbf{0.077$\pm$0.051} & 0.99 \\
\bottomrule
\end{tabular}
\caption{RMSE by contrast family and empirical coverage of node-clustered nominal intervals. Bold marks the lowest RMSE among full-target implementations, and $\dagger$ denotes adaptations. Restricted designs have different targets, and coverage is descriptive.}
\label{tab:main}
\end{table*}

\begin{table}[t]
\centering
\footnotesize
\setlength{\tabcolsep}{1mm}
\begin{tabular}{@{}lcccc@{}}
\toprule
Variant & College & Email & Primary & High \\
\midrule
Full & \textbf{.047$\pm$.011} & \textbf{.052$\pm$.021} & .126$\pm$.045 & \textbf{.083$\pm$.027} \\
$\lambda=1$ & .052$\pm$.022 & .077$\pm$.021 & .206$\pm$.123 & .110$\pm$.040 \\
\midrule
Plug-in & .064$\pm$.019 & .056$\pm$.020 & \textbf{.125$\pm$.045} & .111$\pm$.050 \\
No $M$ & .199$\pm$.005 & .203$\pm$.006 & .231$\pm$.041 & .202$\pm$.010 \\
No $M$ + plug-in & .206$\pm$.007 & .204$\pm$.007 & .231$\pm$.040 & .219$\pm$.017 \\
No-$M$ IPW & .288$\pm$.058 & .318$\pm$.044 & .341$\pm$.111 & .371$\pm$.088 \\
Static exposure & .304$\pm$.025 & .318$\pm$.015 & .309$\pm$.040 & .299$\pm$.026 \\
\bottomrule
\end{tabular}
\caption{Mean $\pm$ SD overall RMSE across seven contrasts. Full denotes DynaNet-DR with its adaptive gate. Bold marks the dataset minimum. The $\lambda=1$ and plug-in rows share the full target; the other variants have target mismatch.}
\label{tab:ablation}
\end{table}

\begin{table}[t]
\centering
\footnotesize
\setlength{\tabcolsep}{1mm}
\begin{tabular}{@{}lccc@{\hspace{1.5mm}}ccc@{}}
\toprule
& \multicolumn{3}{c}{CollegeMsg} & \multicolumn{3}{c}{email-Eu} \\
\cmidrule(lr){2-4}\cmidrule(lr){5-7}
Method & Default & Outcome & Prop. & Default & Outcome & Prop. \\
\midrule
OR & \shortstack{.064\\(.019)} & \shortstack{.421\\(.047)} & \shortstack{.064\\(.019)}
& \shortstack{.056\\(.020)} & \shortstack{.474\\(.094)} & \shortstack{.056\\(.020)} \\
No-$M$ IPW & \shortstack{.288\\(.058)} & \shortstack{.288\\(.058)} & \shortstack{.305\\(.050)}
& \shortstack{.318\\(.044)} & \shortstack{.318\\(.044)} & \shortstack{.403\\(.042)} \\
TL & \shortstack{.054\\(.022)} & \shortstack{\textbf{.205}\\\textbf{(.069)}} & \shortstack{.052\\(.023)}
& \shortstack{.093\\(.026)} & \shortstack{\textbf{.170}\\\textbf{(.050)}} & \shortstack{.101\\(.030)} \\
GML-DR & \shortstack{.053\\(.023)} & \shortstack{.242\\(.060)} & \shortstack{.051\\(.020)}
& \shortstack{.081\\(.023)} & \shortstack{.222\\(.052)} & \shortstack{.088\\(.027)} \\
DynaNet & \shortstack{\textbf{.047}\\\textbf{(.011)}} & \shortstack{.213\\(.069)} & \shortstack{\textbf{.046}\\\textbf{(.014)}}
& \shortstack{\textbf{.052}\\\textbf{(.021)}} & \shortstack{.277\\(.047)} & \shortstack{\textbf{.053}\\\textbf{(.018)}} \\
\bottomrule
\end{tabular}
\caption{Mean overall RMSE, with SD in parentheses. OR is the full-target plug-in outcome regression, and DynaNet denotes DynaNet-DR. Bold marks each minimum. Outcome and Prop. restrict one nuisance model; only No-$M$ IPW has target mismatch.}
\label{tab:misspec}
\end{table}

\begin{figure*}[t]
\centering
\includegraphics[width=0.70\textwidth]{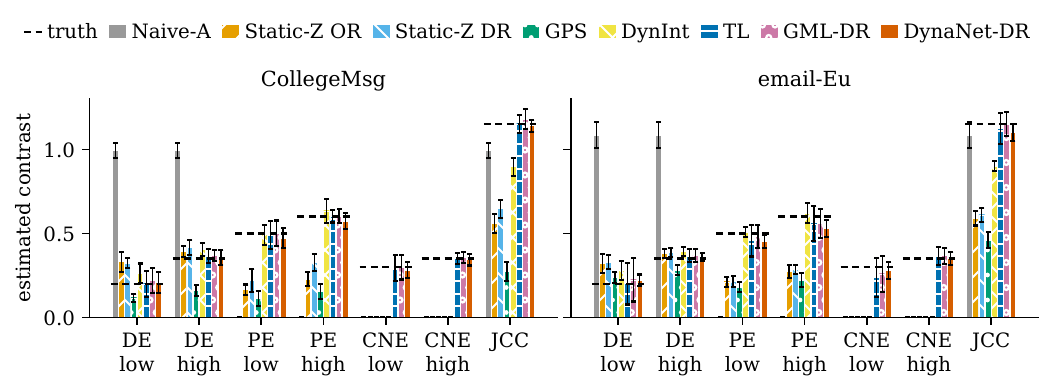}
\caption{Mean estimates and benchmark truths for seven contrasts on CollegeMsg and email-Eu. Error bars are one standard deviation across replications, not confidence intervals. Appendix~\ref{app:contrasts} defines the targets; Appendix~\ref{app:diagnostics} reports the other networks.}
\label{fig:decomp}
\end{figure*}

\begin{figure}[t]
\centering
\includegraphics[width=0.98\columnwidth]{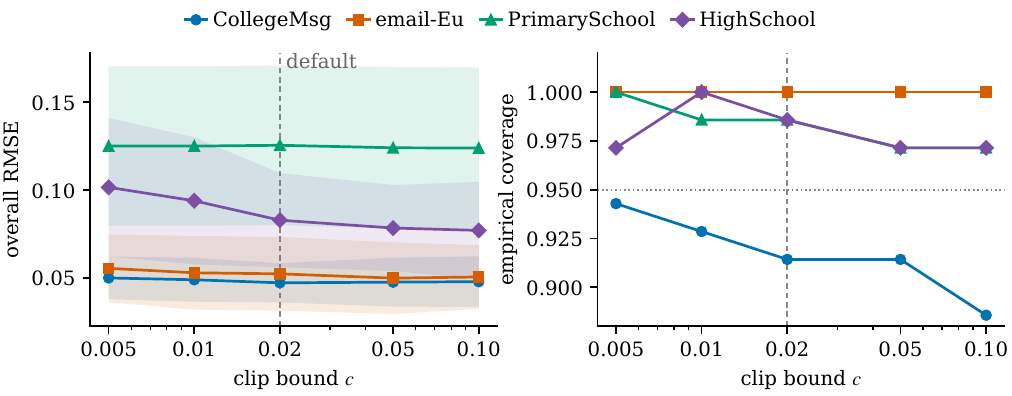}
\caption{DynaNet-DR overall RMSE and empirical coverage by clip bound. Bands show one standard deviation; dashed line marks $c=0.02$.}
\label{fig:clipmain}
\end{figure}

\subsection{Semi-Synthetic Benchmarks on Real Dynamic Networks}

We construct semi-synthetic benchmarks on CollegeMsg \citep{panzarasa2009patterns}, email-Eu \citep{paranjape2017motifs}, and the SocioPatterns Primary School \citep{stehle2011high} and High School 2012 \citep{fournet2014contact} networks. Their fixed graph sequences retain temporal structure, activity heterogeneity, and communities. On each sequence, the generator assigns treatment, derives exposure from assigned peer treatments and observed history, draws a post-assignment evolution-summary label, and generates outcomes with known contrasts; it does not regenerate edges or simulate an alternative graph sequence.

Treatment is logistic in observed history and a generated community-correlated trait $U_i$; $Z_{it}\in\{0,1,2\}$ discretizes a decayed, edge-weighted treated-neighbor share. Binary $M_{i,t+1}$ is logistic in $(A,Z)$, their interaction, history, and $U_i$; the continuous outcome follows the reported model with main effects and interactions. The default nuisance features include $U_i$, whereas the hidden-homophily analysis withholds it. Appendix~\ref{app:impl} gives the generator and dataset statistics.

\subsection{Baselines and Protocol}

Adapted targeted learning (TL; \citealp{chen2024doubly}), the adapted Graph Machine Learning Based Doubly Robust estimator (GML-DR; \citealp{khatami2025graph}), and DynaNet-DR target all reported $(A,Z,M)$ contrasts. The restricted designs are Naive-A \citep{imbens2015causal}, static exposure outcome regression and doubly robust estimators (Static-Z OR/DR; \citealp{robins1986new,bang2005doubly}), an adapted generalized propensity score estimator (GPS; \citealp{forastiere2021identification}), and an adapted dynamic interference model (DynInt; \citealp{lin2026modeling}). Static rows replace dynamic exposure, DynInt uses a representative continuous score, and designs without $M$ return CNE as zero and omit $M$ elsewhere. These rows diagnose target scope rather than same-estimand accuracy. No-$M$ IPW removes both $M$ and the outcome regression; the other ablations remove augmentation, $M$, or dynamic exposure. The $\lambda=1$ row retains the full target and normalized correction while removing only the adaptive gate. All rows use the same analysis population, ten replications, shared available features, and time-ordered 60/20/20 splits. GML-DR also uses embeddings. The joint propensity clip is $[0.02,0.98]$. Appendix~\ref{app:protocol} details the adaptations.

\subsection{Evaluation Metrics}

We report seven contrasts in four labeled families: own-treatment (DE), peer-exposure (PE), controlled network-evolution (CNE), and joint controlled (JCC). Here and in the figures, DE denotes the own-treatment family, not a mediation or path-specific direct effect. Table~\ref{tab:contrasts} lists the target levels and exact truths. For each method and replication $r$, the error on family $\mathcal{F}$ is the root mean squared error (RMSE) against these truths,
\begin{equation}
\mathrm{RMSE}^{(r)}_{\mathcal{F}} = \Bigl(\tfrac{1}{|\mathcal{F}|} \textstyle\sum_{k \in \mathcal{F}} \bigl(\widehat{\tau}^{(r)}_{k} - \tau_{k}\bigr)^{2}\Bigr)^{1/2},
\label{eq:rmse}
\end{equation}
where $k$ indexes contrasts and $r$ indexes replications. We average family RMSE over ten replications. Tables~\ref{tab:ablation} and \ref{tab:misspec} instead pool all seven contrasts. Empirical coverage is the fraction of 70 node-clustered nominal 95 percent intervals containing the truth. Because standard-error constructions differ, including adapted TL's node-clustered dispersion of updated predictions, coverage is descriptive rather than a like-for-like inferential comparison or evidence of validity under evolving-graph dependence. Appendix~\ref{app:diagnostics} reports overlap diagnostics.

\subsection{Main Results (RQ1, RQ2)}

Table~\ref{tab:main} separates same-estimand accuracy from restricted-design scope. Among adapted TL, adapted GML-DR, and DynaNet-DR, which all target $(A,Z,M)$, DynaNet-DR has the lowest unrounded mean RMSE in all 16 dataset-by-family cells; CollegeMsg PE is tied with adapted GML-DR at three-decimal precision. Its unweighted mean cellwise reduction relative to the better of TL and GML-DR is 25.9\%, with larger gains on email-Eu DE (0.026 versus 0.079) and PrimarySchool CNE (0.096 versus 0.192). These results concern the stated adaptations, not the original procedures.

Five restricted rows return CNE as zero, and Naive-A returns PE as zero. DynInt is lowest in two displayed cells, but its representative continuous score and omission of $M$ preclude full-profile comparison. DynaNet-DR's nominal intervals attain empirical coverage from 0.91 to 1.00.

The fixed $\lambda=1$ ablation holds the target, nuisance learners, clipping, splits, and interval construction fixed, so its comparison with Full isolates the adaptive gate within the reported pipeline. Full has lower mean overall RMSE on all four datasets, although the CollegeMsg difference is small and mixed across replications. Relative to plug-in prediction, the implemented correction lowers RMSE on three datasets and is nearly unchanged on PrimarySchool. This benchmark comparison does not establish that the gate is optimal beyond the evaluated settings. Appendices~\ref{app:diagnostics} and \ref{app:impl} report the remaining profiles and paired Wilcoxon tests.

\subsection{Nuisance-Model Stress Tests (RQ3)}

Table~\ref{tab:misspec} reports one-at-a-time nuisance-model restrictions for CollegeMsg and email-Eu. DynaNet-DR improves on the plug-in estimator under the outcome restriction but does not uniformly outperform TL or GML-DR; it performs best under the propensity restriction. These finite-sample stress tests do not verify Theorem~\ref{thm:dr}; Appendix~\ref{app:impl} specifies the restrictions.

As shown in Figure~\ref{fig:clipmain}, DynaNet-DR changes little on three datasets across the clip range. Appendix~\ref{app:clip} reports greater sensitivity for the restricted No-$M$ IPW reference. This does not validate the propensity-correct branch under active clipping; Appendix~\ref{app:hidden} reports hidden-homophily sensitivity.

\subsection{Illustrative MathOverflow Application}

We use MathOverflow \citep{paranjape2017motifs} as an observational illustration. In 30-day bins, $A$ indicates whether a user answers at least one question, $Z$ is the discretized decayed share of answering neighbors, $M$ indicates whether the user interacts with at least one previously unseen neighbor at $t+1$, and $Y=\log(1+\text{weighted degree})$ at $t+2$. Adjustment uses prior network and activity summaries, but some variables share a bin with $A$ and $Z$, leaving within-bin order unresolved; Appendix~\ref{app:impl} gives details.

Averaged across nuisance-model replications, the two estimates in each coordinate-wise family are DE 0.00 and 0.14, PE $-0.04$ and $-0.04$, and CNE 0.37 and 0.46; JCC is 0.44, and the separate naive difference is 0.95. Without ground truth or a hidden-confounding sensitivity analysis, we interpret them as descriptive, representative-score, model-adjusted contrasts. They are not additive components, and the application is not causal validation.

\section{Related Work}

\paragraph{Network causal inference.}
Exposure mappings formalize interference, while network experiments address it by design \citep{hudgens2008toward,aronow2017estimating,eckles2017design}. Observational identification relies on strong adjustment assumptions \citep{shalizi2011homophily,forastiere2021identification}. Graph-learning and dynamic methods learn exposure mappings or use evolving graphs as predictors \citep{ma2021causal,adhikari2025learning,lin2026modeling}; recent work extends doubly robust estimation to networks \citep{chen2024doubly,khatami2025graph}. We add a post-assignment summary while retaining pre-assignment history; Appendix~\ref{app:related} gives details.

\section{Conclusion}

DynaNet-DR estimates a controlled contrast profile using a temporally factorized propensity and normalized augmentation. Its canonical estimator is doubly robust under the stated conditions; stochastic graph interventions and cross-node inference remain future work.

\bibliography{dynanet_dr_v3}

\begin{thebibliography}{25}
\providecommand{\natexlab}[1]{#1}

\bibitem[{Adhikari, Medya, and Zheleva(2025)}]{adhikari2025learning}
Adhikari, S.; Medya, S.; and Zheleva, E. 2025.
\newblock Learning exposure mapping functions for inferring heterogeneous peer
  effects.
\newblock arXiv:2503.01722.

\bibitem[{Aronow and Samii(2017)}]{aronow2017estimating}
Aronow, P.~M.; and Samii, C. 2017.
\newblock Estimating average causal effects under general interference, with
  application to a social network experiment.
\newblock \emph{The Annals of Applied Statistics}, 11(4): 1912--1947.

\bibitem[{Bang and Robins(2005)}]{bang2005doubly}
Bang, H.; and Robins, J.~M. 2005.
\newblock Doubly robust estimation in missing data and causal inference models.
\newblock \emph{Biometrics}, 61(4): 962--973.

\bibitem[{Chen, Hu, and Jiang(2026)}]{chen2026connections}
Chen, S.; Hu, J.; and Jiang, Z. 2026.
\newblock Connections as treatment: causal inference with edge interventions in
  networks.
\newblock arXiv:2601.07267.

\bibitem[{Chen et~al.(2024)Chen, Cai, Yang, Qiao, Yan, Li, and
  Hao}]{chen2024doubly}
Chen, W.; Cai, R.; Yang, Z.; Qiao, J.; Yan, Y.; Li, Z.; and Hao, Z. 2024.
\newblock Doubly robust causal effect estimation under networked interference
  via targeted learning.
\newblock In \emph{Proceedings of the 41st International Conference on Machine
  Learning}, volume 235 of \emph{Proceedings of Machine Learning Research},
  6457--6485. PMLR.

\bibitem[{Comola and Prina(2021)}]{comola2021treatment}
Comola, M.; and Prina, S. 2021.
\newblock Treatment effect accounting for network changes.
\newblock \emph{The Review of Economics and Statistics}, 103(3): 597--604.

\bibitem[{Eckles, Karrer, and Ugander(2017)}]{eckles2017design}
Eckles, D.; Karrer, B.; and Ugander, J. 2017.
\newblock Design and analysis of experiments in networks: reducing bias from
  interference.
\newblock \emph{Journal of Causal Inference}, 5(1): 20150021.

\bibitem[{Forastiere, Airoldi, and Mealli(2021)}]{forastiere2021identification}
Forastiere, L.; Airoldi, E.~M.; and Mealli, F. 2021.
\newblock Identification and estimation of treatment and interference effects
  in observational studies on networks.
\newblock \emph{Journal of the American Statistical Association}, 116(534):
  901--918.

\bibitem[{Fournet and Barrat(2014)}]{fournet2014contact}
Fournet, J.; and Barrat, A. 2014.
\newblock Contact patterns among high school students.
\newblock \emph{PLoS ONE}, 9(9): e107878.

\bibitem[{Hern{\'a}n and Robins(2020)}]{hernan2020causal}
Hern{\'a}n, M.~A.; and Robins, J.~M. 2020.
\newblock \emph{Causal inference: what if}.
\newblock Boca Raton: Chapman \& Hall/CRC.

\bibitem[{Hudgens and Halloran(2008)}]{hudgens2008toward}
Hudgens, M.~G.; and Halloran, M.~E. 2008.
\newblock Toward causal inference with interference.
\newblock \emph{Journal of the American Statistical Association}, 103(482):
  832--842.

\bibitem[{Imbens and Rubin(2015)}]{imbens2015causal}
Imbens, G.~W.; and Rubin, D.~B. 2015.
\newblock \emph{Causal inference for statistics, social, and biomedical
  sciences: an introduction}.
\newblock Cambridge: Cambridge University Press.

\bibitem[{Khatami et~al.(2025)Khatami, Parikh, Chen, Roy, and
  Salimi}]{khatami2025graph}
Khatami, S.~B.; Parikh, H.; Chen, H.; Roy, S.; and Salimi, B. 2025.
\newblock Graph machine learning based doubly robust estimator for network
  causal effects.
\newblock In \emph{Proceedings of the 28th International Conference on
  Artificial Intelligence and Statistics}, volume 258 of \emph{Proceedings of
  Machine Learning Research}, 4366--4374. PMLR.

\bibitem[{Leung(2020)}]{leung2020treatment}
Leung, M.~P. 2020.
\newblock Treatment and spillover effects under network interference.
\newblock \emph{The Review of Economics and Statistics}, 102(2): 368--380.

\bibitem[{Lin et~al.(2026)Lin, Bao, Takeuchi, Cui, and
  Kashima}]{lin2026modeling}
Lin, X.; Bao, H.; Takeuchi, K.; Cui, Y.; and Kashima, H. 2026.
\newblock Modeling dynamic interference for treatment effect estimation from
  dynamic graphs.
\newblock \emph{ACM Transactions on Knowledge Discovery from Data}, 20(3):
  45:1--45:28.
\newblock {DOI}: 10.1145/3796240.

\bibitem[{Ma and Tresp(2021)}]{ma2021causal}
Ma, Y.; and Tresp, V. 2021.
\newblock Causal inference under networked interference and intervention policy
  enhancement.
\newblock In \emph{Proceedings of the 24th International Conference on
  Artificial Intelligence and Statistics}, volume 130 of \emph{Proceedings of
  Machine Learning Research}, 3700--3708. PMLR.

\bibitem[{Manski(1993)}]{manski1993identification}
Manski, C.~F. 1993.
\newblock Identification of endogenous social effects: the reflection problem.
\newblock \emph{The Review of Economic Studies}, 60(3): 531--542.

\bibitem[{Ogburn et~al.(2024)Ogburn, Sofrygin, D{\'i}az, and van~der
  Laan}]{ogburn2024causal}
Ogburn, E.~L.; Sofrygin, O.; D{\'i}az, I.; and van~der Laan, M.~J. 2024.
\newblock Causal inference for social network data.
\newblock \emph{Journal of the American Statistical Association}, 119(545):
  597--611.

\bibitem[{Panzarasa, Opsahl, and Carley(2009)}]{panzarasa2009patterns}
Panzarasa, P.; Opsahl, T.; and Carley, K.~M. 2009.
\newblock Patterns and dynamics of users' behavior and interaction: network
  analysis of an online community.
\newblock \emph{Journal of the American Society for Information Science and
  Technology}, 60(5): 911--932.

\bibitem[{Paranjape, Benson, and Leskovec(2017)}]{paranjape2017motifs}
Paranjape, A.; Benson, A.~R.; and Leskovec, J. 2017.
\newblock Motifs in temporal networks.
\newblock In \emph{Proceedings of the Tenth ACM International Conference on Web
  Search and Data Mining}, 601--610.

\bibitem[{Robins(1986)}]{robins1986new}
Robins, J. 1986.
\newblock A new approach to causal inference in mortality studies with a
  sustained exposure period: application to control of the healthy worker
  survivor effect.
\newblock \emph{Mathematical Modelling}, 7(9-12): 1393--1512.

\bibitem[{S{\"a}vje(2024)}]{savje2024causal}
S{\"a}vje, F. 2024.
\newblock Causal inference with misspecified exposure mappings: separating
  definitions and assumptions.
\newblock \emph{Biometrika}, 111(1): 1--15.

\bibitem[{Shalizi and Thomas(2011)}]{shalizi2011homophily}
Shalizi, C.~R.; and Thomas, A.~C. 2011.
\newblock Homophily and contagion are generically confounded in observational
  social network studies.
\newblock \emph{Sociological Methods \& Research}, 40(2): 211--239.

\bibitem[{Stehl{\'e} et~al.(2011)Stehl{\'e}, Voirin, Barrat, Cattuto, Isella,
  Pinton, Quaggiotto, Van~den Broeck, R{\'e}gis, Lina, and
  Vanhems}]{stehle2011high}
Stehl{\'e}, J.; Voirin, N.; Barrat, A.; Cattuto, C.; Isella, L.; Pinton, J.-F.;
  Quaggiotto, M.; Van~den Broeck, W.; R{\'e}gis, C.; Lina, B.; and Vanhems, P.
  2011.
\newblock High-resolution measurements of face-to-face contact patterns in a
  primary school.
\newblock \emph{PLoS ONE}, 6(8): e23176.

\bibitem[{van~der Laan and Rubin(2006)}]{vanderlaan2006targeted}
van~der Laan, M.~J.; and Rubin, D. 2006.
\newblock Targeted maximum likelihood learning.
\newblock \emph{The International Journal of Biostatistics}, 2(1): Article 11.

\end{thebibliography}

\end{document}